\documentclass[letterpaper,twoside]{article}

\usepackage[accepted]{aistats2021}

\usepackage{graphicx}
\usepackage[ruled,linesnumbered]{algorithm2e}

\usepackage[round,sort&compress]{natbib}

\newcommand{\Bell}{\mathfrak{L}}
\newcommand{\BellT}{\mathfrak{T}}
\newcommand{\eqdef}{=} 
\newcommand{\relname}[2]{\stackrel{\text{(#1)}}{#2}}
\newcommand{\zero}{\bm{0}}

\usepackage{mathtools}
\usepackage{nicefrac} 
\usepackage{xfrac}

\usepackage{multicol}
\usepackage{lipsum}
\usepackage{tikz}
\usetikzlibrary{positioning,angles,quotes,arrows,automata}

\usepackage{float}
\usepackage{bm}
\usepackage{amsfonts}
\usepackage{amsbsy}
\usepackage{amsthm}
\usepackage{color}
\usepackage{amsmath}
\usepackage{enumerate}
\usepackage{amssymb}
\usepackage{enumitem}
\usepackage{array}
\usepackage{nicefrac}

\usepackage{tabularx}
\usepackage{booktabs}
\usepackage{multirow}
\newcolumntype{Y}{>{\centering\arraybackslash}X}

\usepackage{color, colortbl}
\definecolor{Gray}{gray}{0.98}
\definecolor{LightCyan}{rgb}{0.88,1,1}
\newcolumntype{g}{>{\columncolor{Gray}}c}

\DeclareMathOperator*{\argmin}{argmin}

\DeclarePairedDelimiter\abs{\lvert}{\rvert}%
\DeclarePairedDelimiter\norm{\lVert}{\rVert}%
\makeatletter
\let\oldabs\abs
\def\abs{\@ifstar{\oldabs}{\oldabs*}}
\let\oldnorm\norm
\def\norm{\@ifstar{\oldnorm}{\oldnorm*}}
\makeatother

\renewcommand{\*}[1]{{\pmb{#1}}}
\newcommand{\eat}[1]{}
\newcommand{\states}{\mathcal{S}}
\newcommand{\actions}{\mathcal{A}}

\newcommand*{\tr}{^{\mkern-1.5mu\mathsf{T}}}

\newcommand{\E}{\mathbb{E}}
\renewcommand{\P}{\mathbb{P}}

\newenvironment{mprog}{\begin{array}{>{\displaystyle}l>{\displaystyle}l>{\displaystyle}l}}{\end{array}}
\newcommand{\stc}{\\[1ex]  \operatorname{subject\,to} &}
\newcommand{\cs}{\\[1ex] & }
\newcommand{\minimize}[1]{\operatorname*{minimize}_{#1} &}

\newcommand{\one}{\bm{1}}
\newcommand{\Real}{\mathbb{R}}
\renewcommand{\ss}{\;\vert\;}

\newcommand{\zeros}{\bm{0}}
\newcommand{\linf}{L_\infty}
\newcommand{\lw}{_{1,\*w}}
\newcommand{\liw}{_{\infty,\*w}}

\newcommand{\real}{\mathbb{R}}
\newcommand{\prob}{\mathbb{P}}
\newcommand{\simplexs}{\Delta ^{ S }}

\newcommand{\ambset}{\mathcal{P}}
\newcommand{\expect}{\mathbb{E}}
\newcommand{\sa}{_{s,a}}

\newcommand{\qeu}{\mathcal{Q}}

\newcommand{\opt}{^\star}

\DeclarePairedDelimiter\ceil{\lceil}{\rceil}

\DeclareMathOperator{\Med}{Med}

\newcommand{\marek}[1]{\textcolor{blue}{[#1]}}
\newcommand{\bahram}[1]{\textcolor{red}{[#1]}}

\usepackage{thmtools}
\usepackage{nameref}

\usepackage[noabbrev,capitalize]{cleveref}

\theoremstyle{plain}
\newtheorem{theorem}{Theorem}[section]
\newtheorem{lemma}[theorem]{Lemma}

\newtheorem{proposition}[theorem]{Proposition}

\theoremstyle{defintion}

\newtheorem{assumption}{Assumption}
\newtheorem{example}[theorem]{Example}

\theoremstyle{remark}
\newtheorem{remark}[theorem]{Remark}

\makeatletter
\newcommand{\myref}[1]{\cref{#1}\mynameref{#1}{\csname r@#1\endcsname}}
\newcommand{\Myref}[1]{\Cref{#1}\mynameref{#1}{\csname r@#1\endcsname}}

\def\mynameref#1#2{%
	\begingroup
	\edef\@mytxt{#2}%
	\edef\@mytst{\expandafter\@thirdoffive\@mytxt}%
	\ifx\@mytst\empty\else
	\space(\nameref{#1})\fi
	\endgroup
}
\makeatother


\newcommand{\cond}{\,\vert\,}   
\renewcommand{\cite}[1]{\citep{#1}}

\crefmultiformat{assumption}{Assumptions~#2#1#3}{ and~#2#1#3}{, #2#1#3}{ and~#2#1#3}

\begin{document}

\twocolumn[
	\aistatstitle{Optimizing Percentile Criterion Using Robust MDPs}

	\aistatsauthor{ Bahram Behzadian$^{1\star}$   \And Reazul Hasan Russel$^{1\star}$  \And Marek Petrik$^{1}$ \And Chin Pang Ho$^{2}$}

	\aistatsaddress{ $^{1}$University of New Hampshire \And $^{2}$City University of Hong Kong } ]

\begin{abstract}
  We address the problem of computing reliable policies in reinforcement learning problems with limited data. In particular, we compute policies that achieve good returns with high confidence when deployed. This objective, known as the \emph{percentile criterion}, can be optimized using Robust MDPs~(RMDPs). RMDPs generalize MDPs to allow for uncertain transition probabilities chosen adversarially from given ambiguity sets. We show that the RMDP solution's sub-optimality depends on the spans of the ambiguity sets along the value function. We then propose new algorithms that minimize the span of ambiguity sets defined by weighted $L_1$ and $L_\infty$ norms. Our primary focus is on Bayesian guarantees, but we also describe how our methods apply to frequentist guarantees and derive new concentration inequalities for weighted $L_1$ and $L_\infty$ norms. Experimental results indicate that our optimized ambiguity sets improve significantly on prior construction methods.
\end{abstract}

\section{Introduction} \label{sec:introduction}
Applying reinforcement learning to problem domains that involve high-stakes decisions, such as medicine or robotics, demands that we have high confidence in the quality of a policy before deploying it. Markov Decision Processes~(MDPs) represent a well-established model in reinforcement learning~\citep{puterman2005,sutton2018reinforcement}, but their sequential nature makes them particularly sensitive to parameter errors, which can quickly accumulate~\citep{Mannor2007,Xu2009,Tirizoni2018}. Parameter errors are unavoidable when estimating MDPs from data~\cite{Laroche2019}. We focus on computing policies that maximize high-confidence return guarantees in the batch settings. Such guarantees reduce the chance of disappointing the stakeholders after deploying the policy and give them a choice to gather more data or switch to an alternative strategy~\cite{petrik2016safe}.

We propose a new method for computing reliable policies that achieve, with high confidence, good returns once deployed. This objective is also known as the \emph{percentile criterion}~\cite{Delage2009} and can be modeled as risk-aversion to epistemic uncertainty~\cite{petrik2019beyond}. Because optimizing the percentile criterion is NP-hard~\cite{Delage2009}, we use Robust MDPs~(RMDPs)~\cite{Iyengar2005} to optimize it approximately. We establish new error bounds on the performance loss of the RMDPs' policy compared to the optimal percentile solution. Using these new bounds when constructing the RMDPs leads to policies with significantly better return guarantees than reported in prior work~\cite{Delage2009,petrik2019beyond}.

RMDPs generalize MDPs to allow for uncertain, or unknown, transition probabilities~\citep{Nilim2005,Iyengar2005,Wiesemann2013}. Transition probabilities are hard to estimate from data, and even small errors significantly impact the returns and policies. RMDPs consider transition probabilities to be chosen adversarially from a so-called \emph{ambiguity set} (or an uncertainty set). The optimal policy is computed by solving a specific zero-sum game in which the agent chooses the best policy, and an adversarial nature chooses the worst transition probabilities from the ambiguity sets. RMDPs are tractable when their ambiguity sets satisfy so-called rectangularity assumptions~\cite{Wiesemann2013,Mannor2016,Goyal2018}.

Given the goal is to optimize the percentile criterion, the critical question is how to construct the ambiguity sets from state transition samples to optimize the percentile criterion. Prior work constructs ambiguity sets as confidence regions bounded by a distance from a nominal (expected) transition probability~\citep{petrik2016safe,petrik2019beyond,Auer2009,Strehl2004,Gupta2019,Iyengar2005}. In most cases, the ambiguity sets are represented as $L_1$-norm (also referred to as total variation) balls around the nominal probability. In comparison with other probability distance measures, like KL-divergence, the polyhedral nature of the $L_1$-norm allows more efficient computation~\cite{Ho2018}.

The main contribution of this paper is a new technique for optimizing the \emph{shape} of ambiguity sets in RMDPs. Prior work simply constructs ambiguity sets with the smallest size, or volume, that is sufficient to provide the desired high-confidence guarantees. Our new bounds show that the \emph{span} of the ambiguity set along a specific direction is much more important than its volume. To minimize their span, we consider asymmetric ambiguity sets defined in terms of weighted $L_1$ and $L_\infty$ balls. Recent results shows that RMDPs with such ambiguity sets can be solved very efficiently~\citep{Ho2018,Ho2020}. Although our primary focus is on the Bayesian setup, we also discuss the frequentist setup and derive new high-confidence concentration inequalities for the weighted $L_1$ and $L_\infty$ norms.

The remainder of the paper is organized as follows. We first describe the necessary background in \cref{sec:framework} and bound the performance loss of RMDPs as a function of the ambiguity sets' span in \cref{sec:overall_framework}. \cref{sec:shape} describes algorithms that minimize the span of ambiguity sets by optimizing the weights of the norms used in their definition. Then, \cref{sec:size} describes methods for choosing the size of the weighted-norm ambiguity sets. In \cref{sec:frequentist}, we outline the approach in the frequentist setup and present new concentration inequalities for weighted $L_1$ and $\,L_\infty$ ambiguity sets. Finally, the experimental results in \cref{sec:experiments} show that minimizing ambiguity sets' span greatly improves the RMDPs' solution quality.

\emph{Notation}: Bold letters, like $\*x_s$, indicate an $s$-th vector, while $y_s$ would indicate the $s$-th element of a vector $\*y$. The symbol $\Delta^N$ denotes the $N$-dimensional probability simplex (non-negative vectors that sum to $1$). We also use $\mathcal{A}^\mathcal{B}$ to denote the set of all functions $\mathcal{A} \to \mathcal{B}$.

\section{Framework and Related Work} \label{sec:framework}

We consider the standard infinite-horizon MDP setting with finite states $\states = \{1, \ldots, S \}$ and actions $\actions = \{1, \ldots, A\}$. The agent can take any action $a \in \actions$ in every state $s \in \states$ and transitions to the next state $s'$ according to the \emph{true} transition function $P\opt : \states\times\actions\to\simplexs$, where $\simplexs$ is a probability simplex. For any transition function $P : \states\times\actions\to\simplexs$, we use the shorthand $\*p_{s,a} = P(s,a)$ to denote the vector of transition probabilities from a state $s\in\states$ and an action $a\in\actions$. The agent also receives a reward $r_{s,a,s'}\in\Real$; we use $\*r_{s,a} = (r_{s,a,s'})_{s'\in\states} \in \Real^S$ to denote the vector of rewards. The goal is to compute a deterministic policy $\pi: \states \rightarrow \actions$ that maximizes the $\gamma$-discounted return~\citep{puterman2005}:
\[\max_{\pi \in \Pi} \; \rho(\pi,P) = \max_{\pi \in \Pi} \;  \expect \left[ \sum_{t=0}^{\infty} \gamma ^t \cdot r_{S_t, \pi(S_t), S_{t+1}} \right]~,\]
where $S_0 \sim \*p_0$, $S_{t+1} \sim P\opt(S_t, \pi(S_t))$, $\*p_0\in\Delta^S$ is the initial probability distribution, and $\Pi$ is the set of all deterministic policies. The return function $\rho$ is parameterized by $P$, because we assume them to be uncertain or unknown.

We consider the batch RL setting in which the transition function must be estimated from a fixed dataset $D = \left(s_t, a_t, s_{t}'\right)_{t = 1, \ldots, T}$ generated by a behavior policy. We describe the Bayesian setup first and outline the frequentist extension in \cref{sec:frequentist}. Bayesian techniques start with a prior distribution over the transition function $P\opt$ and then derive a posterior distribution $f$ over $P\opt$~\citep{Delage2009,Xu2009,Gelman2014}. We use the concise notation $\tilde{P} = P\opt \cond D$ to represent the posterior over the transition function conditioned on the data $D$. In other words, $\E[\tilde{P}] = \E[P\opt \cond D]$.

\paragraph{Percentile citerion} The Bayesian \emph{percentile criterion} optimization simultaneously optimizes for the policy $\pi$ and a \emph{high-confidence lower bound} on its performance $y$:
\begin{equation}\label{eq:percentile_Bayesian}
	\max_{\pi\in\Pi} \max_{y\in\Real} \; \left\{ y \ss \P_{\tilde{P} \sim f}\left[\rho(\pi,\tilde{P}) \ge y \right] \ge 1-\delta \right\} ~.
\end{equation}
The confidence parameter $\delta \in [0, \nicefrac{1}{2})$ bounds the probability that the optimized policy $\pi$ fails to achieve a return of at least $y$ when deployed. For example, $\delta = 0$ maximizes the worst-case return, and $\delta = 0.5$ maximizes the median return. It is common in practice to choose a small positive value, such as $\delta = 0.05$, in order to achieve meaningful guarantees without being overly conservative. Also, the constraint $\delta < \nicefrac{1}{2}$ is important as our results (\cref{thm:sub_optimality}) do not hold for the risk-seeking setting with $\delta \ge \nicefrac{1}{2}$.

There are several important practical advantages to optimizing the percentile criterion instead of the average return~\cite{Delage2009}. First, the output policy is more robust and less likely to fail catastrophically due to model errors. Second, the objective value $y$ in \eqref{eq:percentile_Bayesian} provides a high-confidence lower bound on the true return. Having such a guarantee on its return helps to avoid an unpleasant surprise when the policy $\pi$ is deployed. If the guarantee $y$ is insufficiently low, the stakeholder may decide to collect more data or choose a different methodology for guiding their decisions.

We emphasize that we develop algorithms that are independent of how the posterior distribution $f$ is computed. Bayesian priors can be as simple as  independent Dirichlet distributions over $\*p\opt_{s,a}$ for each state $s$ and action $a$. However, hierarchical Bayesian models are more practical since they can generalize among states even when $|D| \ll S$~\citep{Delage2009,petrik2019beyond}. Many tools, such as Stan~\cite{Stan2017} or JAGS, now exist that allow for convenient and efficient computation of the posterior distribution $f$ using MCMC.


\paragraph{Robust MDPs} Because the optimization in \eqref{eq:percentile_Bayesian} is NP-hard~\cite{Delage2009}, we seek new algorithms that can approximate it efficiently.
Robust MDPs~(RMDPs), which extend regular MDPs, are a convenient and powerful framework that can be used to optimize the percentile criterion. In particular, RMDPs allow for a generic ambiguity set $\hat\ambset\subseteq \left\{ P : \states\times\actions\to\simplexs \right\}$ of possible transition functions instead of a single known value $P$. The solution to an RMDP is the best policy for the worst-case plausible transition function:
\begin{equation}\label{eq:rmdp}
	\max_{\pi \in \Pi} \min_{P \in \hat\ambset} \, \rho (\pi , P) ~.
\end{equation}
The optimization problem in~\eqref{eq:rmdp} is NP-hard~\citep{Nilim2005,Wiesemann2013} but is tractable for rectangular ambiguity sets which are defined independently for each state and action~\citep{Iyengar2005,le2007robust}. We, therefore, restrict our attention to SA-rectangular ambiguity sets defined as $p$-norm balls around nominal probability distributions for some $w: \states\times\actions\to \Real_{++}^S$ and $\psi : \states\times\actions \to \Real_+$:
\begin{equation*} \label{eq:amb_set_weighted}
	\ambset(w, \psi) \eqdef \left\{ P \in \mathcal{F} \;\vert\; P(s,a) \in \ambset_{s,a}(w(s,a), \psi(s,a)) \right\},
\end{equation*}
where $\mathcal{F} = (\Delta^S)^{\states\times\actions}$. In the remainder of the paper, we resort to the shorter notation $\*w_{s,a} = w(s,a)$ and $\psi_{s,a} = \psi(s,a)$ when the meaning is obvious from the context. Note that $\hat{\ambset}$ refers to a generic ambiguity set, while $\ambset(w,\psi)$ refers to the specific norm-based one. The ambiguity set $\ambset_{s,a}(\*w, \psi)$ for $s\in\states$, $a\in\actions$, positive weights $\*w \in \real^S_{++}$, and budget $\psi\in\Real_+$ is defined as:
\begin{equation}\label{eq:ambiguity_set_poly}
	\ambset_{s,a}(\*w, \psi) \eqdef \left\{ \*p \in \simplexs ~:~ \norm{\*p - \bar{\*p}\sa }_{\*w} \le \psi \right\},
\end{equation}
where $\bar{\*p}\sa \eqdef \E_{\tilde{P}}\bigl[\tilde{P}(s,a) \bigr]$ is the mean posterior transition probability. The weighted polynomial norms are defined as $\norm{\*y}_{1,\*w}  \eqdef \sum_{i=1}^S w_i \cdot \lvert y_i \rvert$ and  $\norm{\*y}_{\infty,\*w} \eqdef \max\, \{ w_i \cdot \lvert y_i \rvert \ss i\in\states \}$. We use the generic notation $\norm{\cdot}_{\*w}$ in statements that hold for both $\norm{\cdot}_{1,\*w}$ and $\norm{\cdot}_{\infty,\*w}$. The weights $\*w$ in \eqref{eq:ambiguity_set_poly} determine the shape of the ambiguity set, and the budget $\psi$ determines its size.

Note that the parameter $\psi$ in the definition of $\ambset_{s,a}(\*w,\psi)$ is redundant. It can be set to $1$ without loss of generality: $\ambset_{s,a}(\*w, \psi)  = \ambset_{s,a}(\nicefrac{1}{\psi}\cdot \*w, 1)$
when $\psi > 0$. In other words, it is possible to change the size of the ambiguity set solely by scaling the weights $\*w$. To eliminate this redundancy, we assume without loss of generality that the weights of the set are normalized such that $\norm{\*w}_2 = 1$.

In rectangular RMDPs, a unique optimal value function $\hat{\*v}\in\Real^S$ exists and is a fixed point of the robust Bellman operator $\Bell : \Real^S \to \Real^S$ defined for each $s\in\states$ and $\*v \in\Real^S$ as~\cite{Iyengar2005}
\begin{equation}  \label{eq:robust_update}
	(\Bell\, \*v)_s \eqdef \max_{a \in \actions} \min_{\*p \in \hat\ambset_{s,a}} \, \Bigl( \*r_{s,a} + \gamma \cdot \*p\tr \*v \Bigr)~.
\end{equation}
The optimal robust value function can be computed using value iteration, policy iteration, and other methods~\cite{Iyengar2005,Kaufman2013,Ho2020}. The optimal robust policy $\hat\pi:\states\to\actions$ is greedy with respect to the optimal robust value function $\hat{\*v}$, and the robust return can be computed from the value function as~\cite{Ho2020}:
\[
\hat{\rho} \;\eqdef\; \max_{\pi \in \Pi} \min_{P \in \hat\ambset} \; \rho(\pi, P) \;=\; \*p_0\tr \hat{\*v}~.
\]
We will find it convenient to use $\hat{\*z}_{s,a} \in \Real^S,\,s\in\states, a\in\actions$ to denote the vector of values associated with the transitions from the state $s$ and action $a$:
\begin{equation} \label{eq:z_def}
	\hat{\*z}_{s,a} \;\eqdef\; \*r_{s,a} + \gamma \cdot \hat{\*v} ~.
\end{equation}
In the remainder of the paper, we use $\hat{\ambset}$ to denote a generic RMDP ambiguity set and use $\ambset(w,\psi)$ to denote an ambiguity set defined in terms of a weighted norm ball.

\section{RMDPs for Percentile Optimization} \label{sec:overall_framework}

This section describes the general algorithm for constructing RMDP ambiguity sets for optimizing the percentile criterion. We derive new bounds on the safety and optimality of the RMDP solution and propose a new algorithm that optimizes them. The bounds and algorithms in this section are general and are not restricted to norm-based ambiguity sets.

An important assumption, which is used throughout this paper, is that the ambiguity set in the RMDP is constructed to guarantee that it contains the unknown transition probabilities $\tilde{P}$ with a high probability as formalized next.
\begin{assumption} \label{asm:inset}
The  RMDP ambiguity set $\hat{\mathcal{P}}\subseteq \left\{ P : \states\times\actions\to\simplexs \right\}$ satisfies that:
\begin{equation*}
	\P_{\tilde{P}} \bigl[\tilde{P} \in \hat\ambset \bigr] \;\ge\; 1-\delta \,.
\end{equation*}
\end{assumption}
\cref{asm:inset} is common when constructing RMDPs for optimizing the percentile criterion~\cite{petrik2019beyond,Delage2009}. The following theorem shows that \cref{asm:inset} is a sufficient condition for $\hat{\rho}$ to be a lower bound on the true return of the robust policy $\hat{\pi}$. We state the result in terms of a generic ambiguity set $\hat\ambset$.
\begin{theorem} \label{thm:ambiguity_sufficient}
If \cref{asm:inset} holds, then the following inequality is satisfied with probability $1-\delta$:
\[
\hat\rho \;\le\;  \rho(\hat{\pi},\tilde{P}) ~.
\]
\end{theorem}
Please see \cref{app:sec_framework} for the proof.

\cref{thm:ambiguity_sufficient} generalizes Theorem~4.2 in \cite{petrik2019beyond} by relaxing its assumptions. In particular, \cref{asm:inset} allows for non-rectangular ambiguity sets $\hat{\mathcal{P}}$ and does not require the use of a union bound in its construction.


Next, we bound the performance loss of the RMDP policy $\hat\pi$ with respect to the optimal percentile criterion guarantee in \eqref{eq:percentile_Bayesian}. As we show, the quality of the RMDP policy depends not simply on the absolute size of the ambiguity set $\psi$, but on its span along a specific direction. The \emph{span} $\beta_{\*z}^{s,a}(\*w,\psi)$ of an ambiguity set $\ambset_{s,a}(\*w,\psi)$ along a vector $\*z \in \Real^S$ for $s\in\states$ and $a\in\actions$ is defined as:
\[
\beta_{\*z}^{s,a}(\*w,\psi) \eqdef \max_{\*p_1, \*p_2} \; \Bigl\{ (\*p_1 - \*p_2)\tr \*z \;\vert\; \*p_1,\*p_2 \in\ambset_{s,a}(\*w,\psi) \Bigr\}.
\]
The following theorem bounds the performance loss of the RMDP solution when using norm-bounded ambiguity sets. Note that \cref{thm:ambiguity_sufficient} implies that, under \cref{asm:inset}, the RMDP return $\hat\rho$ bounds the true return with high confidence and therefore must be a lower bound on the optimal $y\opt$ in \eqref{eq:percentile_Bayesian}.
\begin{theorem} \label{thm:sub_optimality}
When \cref{asm:inset} holds for $\hat{\ambset} = \ambset(w,\psi),\,w:\states\times\actions\to\Real_{++}^S, \psi: \states\times\actions\to\Real_+$, then the performance loss with respect to $y\opt$ optimal in \eqref{eq:percentile_Bayesian} is:
\[
0\;\le\; y\opt - \hat{\rho}  \;\le\; \frac{1}{1-\gamma} \cdot\max_{s\in\states} \max_{a\in\actions} \; \beta_{\hat{\*z}_{s,a}}^{s,a}(\*w,\psi)~,
\]
where $\hat\rho$ is a function of $w$ and $\psi$.
\end{theorem}
The proof can be found in \cref{app:sec_framework}.



%


\begin{figure}
	\centering
	\includegraphics[width=0.6\linewidth]{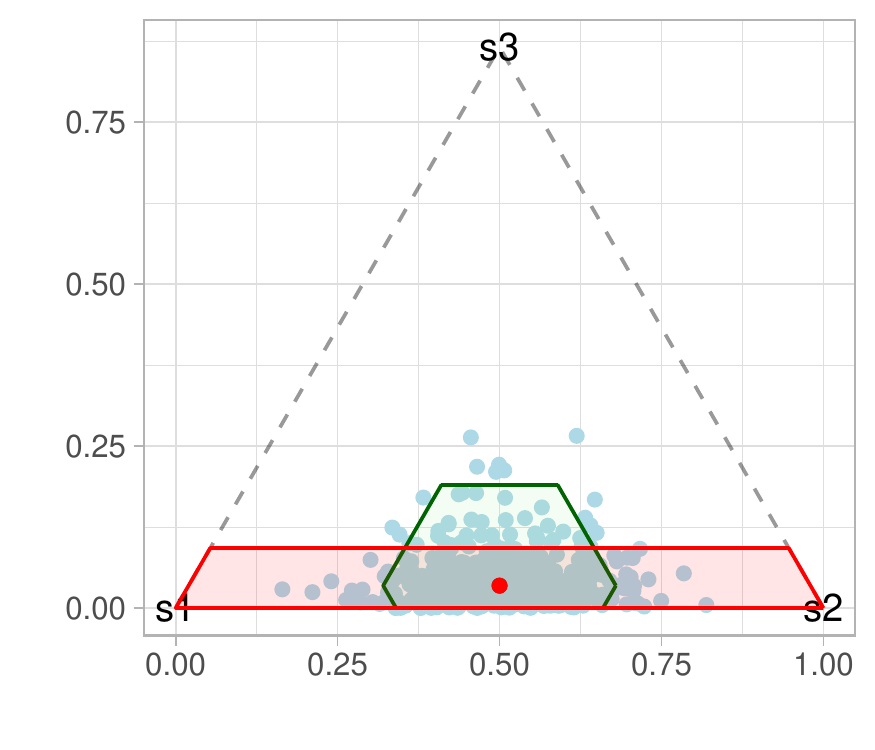}
	\caption{Posterior samples of $\tilde{\*p}$ (blue) and ambiguity sets $\mathcal{P}^{\operatorname{std}}$ (green) and $\mathcal{P}^{\operatorname{opt}}$ (red) from \cref{exm:ambset}.}
	\label{fig:bayes_amb_set}
\end{figure}


The following illustrates how the span along $\hat{\*z}$ impacts the performance loss of the RMDP policy.
\begin{example}\label{exm:ambset}
Consider an MDP with states $\{0,1,2,3\}$ and a single action $\{1\}$. The state $0$ is initial, and the states $1,2,3$ are terminal with $P(i,1,i) = 1,i = 1,2,3$ with zero rewards. To keep the notation simple, we assume that it is only possible to transition from state $0$ to states $1,2,3$. The transition \emph{probability} $\tilde{\*p}_{0,1}$ is uncertain and distributed as $\tilde{\*p}_{0,1}\sim\operatorname{Dirichlet}(10,10,1)$ with $\E[\tilde{\*p}_{0,1}] = [0.48,0.48,0.04]$. The rewards are $\*r_{0,1} = [0.25, 0.25, -1]$. The goal is to maximize the percentile criterion with $\delta = 0.2$.
\end{example}

Take the MDP from \cref{exm:ambset} and construct RMDPs with the following two ambiguity sets depicted in \cref{fig:bayes_amb_set}. Let $\ambset^{\operatorname{std}} = \ambset_{1,1}(\nicefrac{1}{\sqrt{3}}\cdot \one, 0.1)$ be the standard ambiguity set with uniform weights, and let $\ambset^{\operatorname{opt}} = \ambset_{1,1}(\nicefrac{1}{\sqrt{1.12}} \cdot [0.25,0.25,1], 0.1)$ be an ambiguity set with optimized weights $\*w = \nicefrac{1}{\sqrt{1.12}} \cdot [0.25,0.25,1]$. The budgets for both ambiguity sets are minimally sufficient to satisfy \cref{asm:inset}. Intuitively, this means that at least $80\%$ of the posterior samples of $\tilde{\*p}_{0,1}$ (blue dots in \cref{fig:bayes_amb_set}) must be contained inside of each ambiguity set. Now, with $80\%$ confidence, the RMDP with $\ambset^{\operatorname{opt}}$ guarantees return  $\hat{\rho}^{\operatorname{opt}} = 0.16$, while the RMDP with $\ambset^{\operatorname{std}}$ guarantees only $\hat{\rho}^{\operatorname{std}} = -0.06$. Although the volumes of $\ambset^{\operatorname{std}}$ and $\ambset^{\operatorname{opt}}$ are approximately equal, the span along the dimension $\*z = [0.25, 0.25, -1]$ of $\ambset^{\operatorname{opt}}$ is half of the span of $\ambset^{\operatorname{std}}$.





\IncMargin{1.2em}
\begin{algorithm*}
	\KwIn{Confidence $1-\delta$, posterior distribution $f$ over $\tilde{P}$}
	\KwOut{Ambiguity set $\ambset(\*w, \psi)$ }
	\BlankLine
	Compute $\*v' \in \Real^S$ by solving $\max_{\pi} \; \rho\bigl(\pi,\E\bigl[\tilde{P}\bigr]\bigr)$ and
	let $\*z'_{s,a} \gets \*r_{s,a} + \gamma \cdot \*v',\,s\in\states,a\in\actions$\;
 	Compute minimal $\psi': \states\times\actions\to\Real_+$ such that \cref{asm:inset} holds for $\ambset(\nicefrac{1}{\sqrt{S}}\cdot \one,\psi')$\tcp*{\cref{alg:bayes}}
	Compute $\*w_{s,a} \gets \min_{\*w\in\Real^S_+} \; \{ \beta_{\*z'}^{s,a}(\*w, \psi')	\ss \norm{\*w}_2 = 1 \}$ for each $s\in\states$, $a\in\actions$\tcp*{\cref{alg:weight_optimization}}
 	Compute minimal $\psi: \states\times\actions\to\Real_+$ such that \cref{asm:inset} holds for $\ambset(\*w,\psi)$\tcp*{\cref{alg:bayes}}
 	\Return{Ambiguity set $\ambset(\*w, \psi)$}
	\caption{Ambiguity shape optimization scheme.} \label{alg:heuristic}
\end{algorithm*}

Armed with the safety and performance loss guarantees in \cref{thm:ambiguity_sufficient,thm:sub_optimality}, we propose a new heuristic algorithm in \cref{alg:heuristic} which iteratively optimizes the shape of the ambiguity set in order to improve the guaranteed percentile. It constructs ambiguity sets that minimize the span of the ambiguity set. The algorithm may not construct the optimal ambiguity set because it first uses the nominal value function $\*v'$. However, the algorithm provides guarantees  on the quality of the policy that it computes from \cref{asm:inset} and \cref{thm:ambiguity_sufficient,thm:sub_optimality}.


\section{Minimizing Ambiguity Spans} \label{sec:shape}

This section describes tractable algorithms that optimize the weights $\*w$ to minimize that span $\beta_{\*z}^{s,a}$ for some fixed state $s\in\states$, action $a\in\actions$, a vector $\*z \in\Real^S$, and a budget $\psi \in\Real_+$. We describe an analytical solution and a conic formulation that minimize an upper bound on the span for weighted $L_1$ and $L_\infty$ sets. The budget $\psi$ is fixed throughout this section; \cref{sec:size} describes how to optimize it.

The goal of computing the weights $\*w$ that minimize the span of the ambiguity set for a fixed budget $\psi$ can be formalized as the following optimization problem:
\begin{equation} \label{eq:shape_objective}
\min_{\*w\in\Real^S_+} \; \Bigl\{ \beta_{\*z}^{s,a}(\*w, \psi) \ss \norm{\*w}_2 = 1 \Bigr\}.
\end{equation}
The optimization in~\eqref{eq:shape_objective} is not obviously convex, but we propose methods that minimize an \emph{upper} bound on $\beta_{\*z}^{s,a}(\*w, \psi)$. Note that minimizing this upper bound also minimizes an upper bound on \cref{thm:sub_optimality}.

We first describe two analytical solutions and then describe a more precise but also more computationally intensive method based on second order conic approximation. The following lemma provides a bound that enables efficient optimization.
\begin{lemma}\label{thm:choose_weights}
	The span $\beta_{\*z}^{s,a}$ of the ambiguity set $\ambset_{s,a}(\*w,\psi)$ is bounded for any $\lambda\in\Real$ as:
	\begin{align}
		\beta_{\*z}^{s,a}(\*w, \psi)
		%
		\label{eq:upperbound_negativity}
		&\le 2 \cdot \psi \cdot \norm{ \*z - \lambda \cdot \*1}_\star ~,
	\end{align}
	where $\norm{\cdot}_\star$ is the norm dual to $\norm{\cdot}_{\*w}$.
\end{lemma}
The proof is deferred to \cref{app:sec_shape}. Recall that the \emph{dual norm} is defined as $\norm{\*c}_{\star}=\max_{\*x\in \Real^S}\, \left\{\*c\tr \*x \ss \norm{\*x} \le 1 \right\}$.

In order to use the bound in \cref{thm:choose_weights}, we need to derive the dual norms to the weighted $L_1$ and weighted $L_\infty$ norms. For unweighted $p$-norms, it is well known that $L_1$ and $L_\infty$ norms are dual of each other, but we are not aware of a similar result for their weighted variants. The following lemma establishes that weighted $L_1$ and $L_\infty$ norms are dual as long as their weights are inverse elementwise.
\begin{lemma} \label{th:l1_l8} Suppose that $\*w\in\Real^S$ and $\*w'\in\Real^S$ are positive $w_i > 0, w_i' > 0$ and satisfy that $w_i' = \nicefrac{1}{w_i}$ for all $i\in\states$. Then:
\[
\norm{\*z}_{\infty,\*w'} \;=\; \max_{\*x\in\Real^S}\, \left\{ \*z \tr \*x \ss \norm{\*x}\lw = 1 \right\}  .
\]
\end{lemma}
\vspace{-3mm}
The proof of the lemma can be found in \cref{app:sec_shape}.



\begin{algorithm}
	\KwIn{Norm $q\in\{1,\infty\}$, parameter $\lambda\in\Real$}
	\KwOut{Weights $\*w\opt\in\Real_+^S$ that minimize \eqref{eq:upperbound_negativity}}
	\uIf{$q = 1$}{
		$ w_i\opt \gets \frac{ \lvert z_i - \lambda \rvert^{1/3}}{\sqrt{\sum_{j=1}^S {\lvert z_j - \lambda \rvert^{2/3}}} },\, \forall i\in\states $ \;
		}
	\ElseIf{$q = \infty$}{
		$w_i\opt \gets \frac{\lvert z_i - \lambda \rvert}{\sqrt{\sum_{j=1}^S \lvert z_j - \lambda \rvert^2} }, \, \forall i\in\states $ \;
	 }
	\Return $\*w\opt$ \;
	\caption{Weight optimization.} \label{alg:weight_optimization}
\end{algorithm}

Based on the results above, \cref{alg:weight_optimization} summarizes our algorithms for computing weights $\*w$ that minimize the upper bound on the performance loss in \cref{thm:sub_optimality}. The algorithm runs in linear time. Note that the algorithm assumes that a value of $\lambda$ is given. Although it would be possible to optimize for the best value of $\lambda$, our preliminary experimental results suggest that this is not worthwhile because it does not lead to a significant improvement. Instead, we use $\lambda = (\max_i z_i + \min_i z_i) / 2$ and $\lambda = \operatorname{median}(\*z)$ for $L_\infty$ and $L_1$ norms respectively. These are the optimal values (values for which the upper bound is smallest) for the uniform weight version of \eqref{eq:upperbound_negativity}. The following proposition states the correctness of this algorithm.
\begin{proposition} \label{prop:weights_analytical}
Fix an arbitrary $\lambda\in\Real$ and let $\*w\opt\in\Real_+^S$ the return from \cref{alg:weight_optimization}. Then $\*w\opt$ is an optimal solution to \eqref{eq:upperbound_negativity} weighted $L_1$ and $L_\infty$ norms.
\end{proposition}
Please see \cref{app:sec_shape} for the proof.


It is important to recognize that even though \cref{alg:weight_optimization} effectively minimizes the value $\beta^{s,a}_{\*z}$, it may, in the process, violate \cref{asm:inset}. This is because scaling weights may reduce the probability that $\tilde{P}\in\ambset$. We are not aware of a tractable algorithm that can optimize the weights $\*w$ directly while enforcing the constraint of \cref{asm:inset}. Instead, the constraint $\norm{\*w}_2 = \psi$ serves as a proxy to prevent the ambiguity from shrinking. This is why it is necessary to re-optimize the budget $\psi$ in \cref{alg:heuristic} after the weights are optimized.

As an alternative to the analytical algorithms in \cref{alg:weight_optimization}, we also examine a Second-Order Conic Program~(SOCP) formulation. This formulation optimizes a tighter upper bound on $\beta^{s,a}_{\*z}$ but is more computationally intensive. For any fixed state $s$ and action $a$, the following SOCP minimizes the bound \eqref{eq:upperbound_negativity} on $\beta_{\*z}^{s,a}(\*w,\psi)$ for the $L_1$ norm:
\begin{equation} \label{eq:socp}
\begin{mprog}
	\minimize{\*g,c,\lambda} \psi \cdot c 
	\stc \*g \;\ge\; \max\{ \*z - \lambda \cdot \one, -\*z + \lambda \cdot \one\}
	\cs \*g\tr \*g \;\le\; c^2, \quad \*g \;\ge\; \zero~.
\end{mprog}
\end{equation}
The SOCP formulation follows from \cref{th:l1_l8} and variable substitution $\*g = \*w \cdot c$. \\


\begin{remark}[Unreachable states]
We assume that the prior can specify some transitions as impossible, or unreachable: that is $P(s,a,s') = 0$. This information is used as an additional pre-processing step in optimizing the weights. In particular, if the transition from state $s$ after taking action $a$ to state $s'$ is not possible, then we set $(\*w_{s,a})_{s'} = \infty$. Or, in other words, each $\*p \in \ambset_{s,a}(\*w, \psi)$ satisfies $p_{s'} = 0$.
\end{remark}

\section{Minimizing Ambiguity Budgets} \label{sec:size}

This section describes how to determine the size of the ambiguity set in the Bayesian setting in order to minimize the performance loss in \cref{thm:sub_optimality} of the RMDP policy while satisfying \cref{asm:inset}. We assume that the weights $\*w_{s,a}, s\in\states, a\in\actions$ are arbitrary but fixed and aim to construct $\psi_{s,a},s\in\states,a\in\actions$ to minimize the performance loss.

Before describing the algorithm, we state a simple observation that motivates its construction. The following lemma implies that the smaller the ambiguity budget is, the better $\hat\rho$ approximates the percentile criterion. Of course, this is only true as long as the budget is sufficiently large for \cref{asm:inset} to hold. The following proposition follows from the definition of $\beta^{s,a}_{\*z}$ by algebraic manipulation.
\begin{lemma} \label{prop:obvious}
The function $\psi \mapsto \beta^{s,a}_{\*z}(\*w_{s,a}, \psi)$ is non-decreasing.
\end{lemma}

\begin{algorithm}
\KwIn{Posterior samples $P_1, \ldots, P_n$ from $\tilde{P}$,
	weights $\*w_{s,a}$, norm $q\in \{1,\infty\}$}
\KwOut{Nominal $\bar{\*p}_{s,a}$ and budget $\psi_{s,a}$}
Compute nominal $\bar{\*p}_{s,a} \gets (1/n) \sum_{i=1}^n P_i(s,a) $ \;
Compute distance $d_i \gets \norm{P_i(s,a) - \bar{\*p}_{s,a} }_{q,\*w_{s,a}}$ \;
Ascending sort: $d_{(j)} \le d_{(j+1)},\,j=1,\ldots,n$\;
Compute the quantile $\psi_{s,a} \gets d_{(\ceil{(1-\delta/(S\cdot A)) \cdot n})}$ \;
\Return{$\bar{\*p}_{s,a}$ and $\psi_{s,a}$}
\caption{Budget optimization.} \label{alg:bayes}
\end{algorithm}

We are now ready to describe our method as outlined in \cref{alg:bayes}. The algorithm follows the well-known sample average approximation~(SAA) approach common in stochastic programming~\cite{Shapiro2014}. It constructs ambiguity sets as \emph{credible regions} for the posterior distribution over $\tilde{P}$ similarly to prior work~\citep{petrik2019beyond}. The next proposition states the correctness of \cref{alg:bayes}.
\begin{proposition} \label{prop:size_correct}
Suppose that $\psi_{s,a}$ are computed by \cref{alg:bayes} for some $\*w_{s,a}$ for each $s\in\states$ and $a\in\actions$. Also let $w :  (s,a) \mapsto \*w_{s,a}$ and $\psi : (s,a) \mapsto \psi_{s,a}$. Then  $\ambset(w, \psi)$ satisfies \cref{asm:inset} with high probability when a sufficient number of samples from $\tilde{P}$ are used.
\end{proposition}
Please see \cref{app:sec_size} for the proof.


\cref{alg:bayes} constructs credible regions for each state and action separately~\citep{murphy2012machine}. A notable limitation of \cref{alg:bayes} is that it constructs the credible regions independently for each state and action. Although this is convenient computationally, it also means that the confidence region needs to rely on the union bound which makes it the impractical when the number of states and actions is large. Although, \cref{asm:inset} allows for a construction that avoids the union-bound-based construction.

While \cref{prop:size_correct} provides asymptotic convergence guarantees, it is possible to obtain finite sample guarantee by using more careful analysis~\cite{Luedtke2008} or by adapting \cref{alg:bayes} as suggested in \cite{HHL2020}. We leave this finite-sample analysis for future work.

\section{Frequentist Guarantees} \label{sec:frequentist}

In this section, we extend the analysis above to outline how our results apply to frequentist guarantees. The advantage of the frequentist setup is that it provides guarantees even without needing access to a prior distribution. The disadvantage is that, without good priors, frequentist settings may need an excessive amount of data to provide reasonable guarantees. The main contribution in this section are new sampling bounds for weighted $L_1$ and $L_\infty$ ambiguity sets.

The frequentist perspective on the percentile criterion~\cite{Delage2009} represents a viable alternative to the Bayesian perspective when it is difficult to construct a good prior distribution. The frequentist view assumes that the true model $P\opt$ is known. The analysis considers the uncertainty over datasets. To define the criterion, let $\mathcal{D}$ represent the set of all possible datasets $D$. Then the pair of algorithms $F: \mathcal{D}\to\Pi$, which computes the policy for a dataset, and $G: \mathcal{D}\to\Real$, which estimates the return of the policy, solves the percentile criterion if:
\begin{equation} \label{eq:percentile_frequentist}
\P_{D\sim P\opt}\left[\rho(F(D),P\opt) \ge G(D) \right] \;\ge\; 1-\delta ~.
\end{equation}
A frequentist modeler assumes that $P\opt_{s,a}$ is fixed and the probability statements are qualified over sampled data sets $(s_t,a_t,s_t')_{t=1,\ldots,T}$ generated from the true transition probabilities $s_t' \sim \*p\opt_{s_t,a_t}$.

To construct an RMDP that solves the frequentist percentile criterion, we make very similar assumptions to the Bayesian setting. The next assumption restates \cref{asm:inset} in the frequentist setting; note the change in random variables.
\begin{assumption} \label{asm:inset_freq}
	The data-dependent ambiguity set $\hat{\mathcal{P}}$ satisfies:
	\begin{equation*}
		\P_{D \sim P\opt} \bigl[P\opt \in \hat\ambset \bigr] \;\ge\; 1-\delta\,,
	\end{equation*}
	where $\hat\ambset$ is a function of $D$.
\end{assumption}
Recall that \cref{thm:ambiguity_sufficient} establishes that an RMDP that satisfies \cref{asm:inset} computes a~high-confidence lower bound on the return. The proof of \cref{thm:ambiguity_sufficient} easily extends to the frequentist setup. Therefore, \cref{asm:inset_freq} implies that $\P_{D} \left[\hat{\rho} \le \rho(\hat\pi, P) \right] \ge 1-\delta$ where $\hat\rho$ and $\hat{\pi}$ are the return and policy to the RMDP. In other words, the RMDP algorithm (joint policy and return estimate computation) solves the frequentist percentile criterion in \eqref{eq:percentile_frequentist} when \cref{asm:inset_freq} holds.


Because the optimization methods described in \cref{sec:shape} make no probabilistic assumptions, they can be applied to the frequentist setup with no change. The optimization of $\psi$ described in \cref{sec:size} assumes that samples from the posterior over transition functions are available and cannot be readily used to satisfy \cref{asm:inset_freq}. Instead, we present two new finite-sample bounds that can be used to construct frequentist ambiguity sets. Since prior work has been limited to the ambiguity sets defined in terms $L_1$ ambiguity sets with uniform weights~\citep{Weissman2003xx,Auer2010a,Dietterich2013,petrik2019beyond}, we derive new high-confidence bounds for ambiguity sets defined using \emph{weighted} $L_1$ and $L_\infty$ norms. To state our new results, let the nominal point $\bar{\*p}\sa\in\Delta^S$ in \eqref{eq:ambiguity_set_poly} be the empirical estimate of the transition probability computed from $n\sa \in \mathbb{N}$ transition samples for each state $s\in\states$ and action $a\in\actions$.
\begin{theorem}[$L_\infty$ norm] \label{thm:budget_freq_infty}
	 Suppose that $\ambset(\*w,\psi)$ is defined in terms of the $\*w_{s,a}$-weighted $L_\infty$ norm. Then \cref{asm:inset_freq} is satisfied if $\psi_{s,a}\in\Real_+$ for each $s\in\states$ and $a\in\actions$ satisfies the following inequality:
	\begin{equation} \label{eq:w8_error_delta}
		\delta \;\leq\; 2 \cdot S A \cdot \sum_{i=1}^S \exp \left(-2 \frac{\psi\sa^2 \cdot n\sa}{(\*w_{sa})_{i}^2} \right)~.
	\end{equation}
\end{theorem}
\begin{theorem}[$L_1$ norm] \label{thm:budget_freq_l1}
	Suppose that $\ambset(\*w,\psi)$ is defined in terms of the $\*w_{s,a}$-weighted $L_1$ norm.
	Then \cref{asm:inset_freq} is satisfied if $\psi_{s,a}\in\Real_+$ for each $s\in\states$ and $a\in\actions$ satisfies the following inequality:
	\begin{equation} \label{eq:w1_error_delta}
		\delta \;\le\;	2\cdot SA\cdot\sum_{i = 1}^{S-1} 2^{S - i} \cdot \exp  \left(  -  \frac{\psi\sa^2 \cdot n\sa}{2 \cdot (\*w_{sa})_i^2} \right)~,
	\end{equation}
	where positive weights $\*w_{s,a} \in \real_{++}^S, s\in\states, a\in\actions$ are assumed to be sorted in a non-increasing order $(\*w_{s,a})_i \ge (\*w_{s,a})_{i+1}$ for $i = 1, \ldots, S-1$.
\end{theorem}
\vspace{-3mm}
The proofs of the theorems are in \cref{app:sec_freq}. They follow by standard techniques combining the Hoeffding and union bounds.

A natural question is how to construct $\psi_{s,a}$ that satisfies \cref{thm:budget_freq_infty,thm:budget_freq_l1}. Although the theorems do not provide us with an analytical solution, the value of $\psi_{s,a}$ can be computed efficiently using the standard bisection method~\cite{Boyd2004}. This is because right-hand side functions in \eqref{eq:w8_error_delta} and \eqref{eq:w1_error_delta} are monotonically decreasing in $\psi_{s,a},s\in\states,a\in\actions$. \cref{thm:weighted_lone_bern} further tightens the error bounds using Bernstein's inequality.

\cref{thm:budget_freq_infty,thm:budget_freq_l1} also provide new insights into which ambiguity set may be a better fit for a particular problem. Simple algebraic manipulation and \eqref{eq:upperbound_negativity} show that the $L_1$ norm is preferable to the $L_\infty$ norm when $\norm{\*v - \bar{v} \cdot \one}_1 > \sqrt{S} \cdot \norm{\*v - \tilde{v} \cdot \one}_\infty$. Here, $\*v\in\Real^S$ is the optimal value function, $\bar{v} = \one\tr \*v / S$ is the mean value, and $\tilde{v}$ is the median value of $\*v$.

In terms of their tightness, \cref{thm:budget_freq_infty,thm:budget_freq_l1} are similar to the most well-known bounds on the uniformly-weighted norms. \cref{thm:budget_freq_l1} recovers the equivalent best-known (Hoeffding-based) result for uniformly-weighted norm within a factor of $2$. We are not aware of comparable prior results for ambiguity sets defined in terms of $L_\infty$ norms. Unfortunately, frequentist bounds on probability distributions are generally useful only when the number of samples $n_{s,a}$ is quite large. We also investigated Bernstein-based versions of the bounds, but they show little difference in our experimental results.

Finally, it is important to note that \cref{thm:budget_freq_infty,thm:budget_freq_l1} require that the weights $\*w$ are independent of data. Therefore, the weights $\*w$ should be optimized using a dataset different from the one used to estimate $\psi$. However, in our experiment, we found that reusing the same dataset to optimize both $\*w$ and $\psi$ empirically does compromise the percentile guarantees.

\section{Empirical Evaluation}\label{sec:experiments}



%

In this section, we evaluate \cref{alg:heuristic} empirically using five standard reinforcement domains that have been previously used to evaluate robustness.

\cref{tab:bayesian_summary,tab:frequentist_summary} summarize the results for the Bayesian and frequentist setups respectively. The results compare our algorithms (rows) against baselines (rows) for fixed datasets $D$ for all domains (column). The method names indicate how the weights are computed and which norm is used to defined the ambiguity set. Methods denoted as ``Uniform'' represent $\*w  = \one$ and ``Optimized'' represent $\*w$ computed using \cref{alg:heuristic,alg:weight_optimization}. Please see \cref{app:empirical} for a complete report of the statistics and methods (including the SOCP formulation).

As the main metric, we compare the computed return guarantees $\hat\rho$ (the return of the RMDP). Because all methods use ambiguity sets that satisfy \cref{asm:inset,asm:inset_freq}, $\hat\rho$ lower bounds $\rho(\hat\pi, \tilde{P})$ with probability $1-\delta$. In order to enable the comparison of the results among different domains, we normalize the guarantee by the maximal nominal return $\bar\rho = \max_{\pi\in\Pi} \rho(\pi, \E[\tilde{P}])$. We use $\bar\rho$ instead of the unknown $y\opt$.

As a baseline, we compare our results with the standard RMDPs construction~\cite{petrik2019beyond,Delage2009}, which uses uniformly-weighted $L_1$ and $L_\infty$ norms. We do not compare to policy-gradient-style methods in \cite{Delage2009} because they cannot be used with general posterior distributions over $\tilde{P}$ in our domains. We note that various modifications to probability norms have been proposed in the RL context~(e.g.,~\cite{Maillard2014,taleghan2015pac}), but it is unclear how to use them in the context of the percentile criterion.
The results in \cref{tab:bayesian_summary,tab:frequentist_summary} show that optimizing the weights in RMDP ambiguity sets decreases the guaranteed performance loss dramatically in Bayesian settings (geometric mean $2.8\times$) and reliably in frequentist settings (geometric mean $1.6\times$). The guarantees improve because the RMDPs with optimized sets simultaneously compute a better policy and a tighter bound on its return. Note that zero losses in the tables may be unachievable ($\bar{\rho} > y\opt$), and losses greater than one are possible (when $\bar\rho < 0$). The total computational complexity of \cref{alg:heuristic,alg:weight_optimization} is small and reported in \cref{app:empirical}.

We now briefly summarize the domains used; please consult \cref{app:empirical} for more details.

\emph{RiverSwim (RS)} is a simple and standard benchmark~\citep{strehl2008analysis}, which is
an MDP consisting of six states and two actions. The process follows by sampling synthetic datasets from the true model and then computing the guaranteed robust returns for different methods. The prior is a uniform Dirichlet distribution over reachable states.

\emph{Machine Replacement (MR)} is a small benchmark MDP problem with $S=10$ states that models progressive deterioration of a mechanical device~\cite{Delage2009}. Two repair actions $A=2$ are available and restore the machine's state. Uses a Dirichlet prior.

\emph{Population Growth Model (PG)} is an exponential population growth model~\citep{kery2011bayesian}, which constitutes a simple state-space $0,\ldots,S = 50$ with exponential dynamics. At each time step, the land manager has to decide whether to apply a control measure to reduce the species' growth rate. We refer to~\cite{Tirizoni2018}  for more details of the model.

\emph{Inventory Management (IM)} is a classic inventory management problem~\citep{Zipkin200}, with discrete inventory levels $0,\ldots,S=30$. The purchase cost, sale price, and holding cost are $2.49, 3.99$, and $0.03$, respectively. The demand is sampled from a normal distribution with a mean $S/4$ and a standard deviation of $S/6$.
It also uses a Dirichlet prior.

\emph{Cart-Pole (CP)} is the standard RL benchmark problem \citep{sutton2018reinforcement,openaigym}. We collect samples of $100$ episodes from the true dynamics. We fit a linear model with that dataset to generate synthetic samples and aggregate close states to a 200-cell grid ($S=200$) using the k-nearest neighbor strategy and assume a uniform Dirichlet prior.

\begin{table}
\centering
\begin{tabularx}{\linewidth}{lrrrrr}
	\toprule
	         & RS & MR & PG & IM & CP \\
	\midrule
	Uniform $L_1$      & 0.60  & 1.56 &  5.24 &  0.97 & 0.77 \\
	Uniform $L_\infty$ & 0.60  & 1.56 & 5.50 & 0.98 &  0.76\\
	Optimized $L_1$       & 0.25  &  0.41 &  1.84 & 0.90 & 0.12\\
	Optimized $L_\infty$  & 0.31 &  0.39 & 3.10 & 0.87 & 0.19\\
	\bottomrule
\end{tabularx}
\caption{Normalized \emph{Bayesian} performance loss $(\bar{\rho} - \hat{\rho}) / |\bar{\rho}|$ for $\delta = 0.05$. (Smaller value is better).} \label{tab:bayesian_summary}
\end{table}

\begin{table}
	\centering
	\begin{tabularx}{\linewidth}{lrrrrr}
		\toprule
		& RS & MR & PG & IM & CP \\
		\midrule
		Uniform $L_1$        & 0.80 & 5.83 &  5.66 & 1.05 & 0.78\\
		Uniform $L_\infty$   & 0.76 & 3.45 & 5.65 & 1.05  & 0.78\\
		Optimized $L_1$       &  0.53 & 1.05 & 5.55 & 0.99 & 0.77 \\
		Optimized $L_\infty$  & 0.43& 0.94 & 5.56 & 0.96 & 0.69\\
		\bottomrule
	\end{tabularx}
	\caption{Normalized \emph{frequentist} performance loss $(\bar{\rho} - \hat{\rho}) / |\bar{\rho}|$ for $\delta = 0.05$. (Smaller value is better). } \label{tab:frequentist_summary}
\end{table}

\section{Conclusion} \label{sec:conclusion}

We proposed a new approach for optimizing the percentile criterion using RMDPs that goes beyond the conventional ambiguity sets. At the heart of our method are new bounds on the performance loss of the RMDPs with respect to the optimal percentile criterion. These bounds show that the quality of the RMDP is driven by the span of its ambiguity sets along a specific direction. We proposed a linear-time algorithm that minimizes the span of the ambiguity sets and also derived new sampling guarantees. Our experimental results show that this simple RMDP improvement can lead to much better return guarantees. Future work needs to focus on scaling the method to a large state-space using value function approximation or other techniques.

\subsection*{Acknowledgments}

We thank the anonymous reviewers for comments that helped to improve this paper. This work was supported, in part, by the National Science Foundation (Grants IIS-1717368 and IIS-1815275), the CityU Start-up Grant (Project No. 9610481), the CityU Strategic Research Grant (Project No. 7005534), and the National Natural Science Foundation of China (Project No. 72032005). Any opinion, finding, and conclusion or recommendation expressed in this material are those of the authors and do not necessarily reflect the views of the National Science Foundation and the National Natural Science Foundation of China.

\bibliography{bellman}

\newpage
\appendix
\onecolumn

\section{Technical Results and Proofs} \label{technical_proofs}

\subsection{Proofs of Results in \cref{sec:overall_framework}} \label{app:sec_framework}

\begin{proof}[Proof of \cref{thm:ambiguity_sufficient}]
	The result can be derived as:
	\begin{align*}
		\P_{\tilde{P}\sim f}\left[\hat\rho \le  \rho(\hat{\pi},\tilde{P})  \right] &\stackrel{\text{(a)}}{=}
		\P_{\tilde{P}\sim f}\left[\rho(\hat{\pi},\tilde{P}) \ge \max_{\pi\in\Pi} \min_{P \in \hat\ambset} \rho(\pi, P) \right] \\
		&\stackrel{\text{(b)}}{=} \P_{\tilde{P}\sim f}\left[\rho(\hat{\pi},\tilde{P}) \ge  \min_{P \in \hat\ambset} \rho(\hat\pi, P) \right] \\
		&\stackrel{\text{(c)}}{\ge} \P_{\tilde{P}\sim f}\left[\tilde{P} \in \hat\ambset \right] \stackrel{\text{(d)}}{\ge} 1- \delta~.
	\end{align*}
	The equality (a) follows from the definition of $\hat\rho$, the inequality (b) follows from $\hat\pi \in \Pi$ and is optimal, (c) follows because $\rho(\hat{\pi},\tilde{P}) \ge  \min_{P \in \hat\ambset} \rho(\hat\pi, P)$ whenever $\tilde{P} \in \hat\ambset$, and (d) follows from the theorem's hypothesis.
\end{proof}

\begin{proof}[Proof of \cref{thm:sub_optimality}]
	Let $\hat{\ambset} = \ambset(\*w,\psi)$ and let $\hat{\rho}$ and $\hat{\pi}$ be the optimal return and policy for $\hat{\ambset}$ respectively. We start by establishing the following bound:
	\[
	\hat\rho \ge \max_{\pi\in\Pi} \rho(\pi, \tilde{P}) - \frac{\beta_{\hat{\*z}} (\*w, \psi)} {1-\gamma}~,
	\]
	where
	\[
	\beta_{\hat{\*z}}(\*w, \psi) \eqdef \max_{s\in\states} \max_{a\in\actions}\, \beta^{s,a}_{\hat{\*z}}(\*w, \psi) ~.
	\]
	Let $\hat{\*v} \in \Real^S$ be the optimal robust value function that satisfied $\hat{\*v} = \Bell \hat{\*v}$ for the ambiguity set $\hat{\ambset} = \ambset(\*w, \psi)$. We use $\hat{\ambset}$ as a shorthand for $\ambset(\*w, \psi)$ throughout the proof. Recall that $\hat{\rho} = \*p_0\tr \hat{\*v}$. We also use $\BellT_\pi^P$ to represent the Bellman evaluation operator for a policy $\pi\in\Pi$ and a transition function $P$ defined for each $s\in\states$ as:
	\[
	(\BellT_\pi^P\, v )_s = P(s,\pi(s))\tr (\*r_{s,a} + \gamma \cdot \* v)~.
	\]
	It is well known that $\BellT_\pi^P\, v$ is a contraction, is monotone, and has a unique fixed point. Let $\tilde{v}$ be the unique fixed point of $\BellT_{\tilde\pi}^{\tilde{P}}$:
	\[
	\tilde{\*v} = \BellT_{\tilde\pi}^{\tilde{P}} \tilde{\*v}~,
	\]
	where $\tilde\pi \in \arg\max_{\pi\in\Pi} \, \rho(\pi, \tilde{P})$. Note that it is well known that:
	\[
	\*p_0\tr \tilde{\*v} = \rho(\tilde{\pi}, \tilde{P})~.
	\]
	Now suppose that $\tilde{P} \in \hat\ambset$, which holds with probability $1-\delta$ according to \cref{asm:inset}. Then it is easy to see that:
	\[
	\*p_0\tr \hat{\*v} = \min_{P\in\hat\ambset} \rho(\pi, P) \le \rho(\pi, \tilde{P}) \le \*p_0\tr \tilde{\*v}~.
	\]
	Therefore:
	\[
	0 \le \*p_0\tr \tilde{\*v} - \*p_0\tr \hat{\*v} \le \norm{\tilde{\*v} - \hat{\*v}}_\infty~.
	\]
	We are now ready to establish the probabilistic bound which is based on bounding the Bellman residual as follows:
	\begin{align*}
	(\BellT_{\tilde\pi}^{\tilde{P}} \hat{\*v} - \hat{\*v} )_s &\relname{a}{=} 	(\BellT_{\tilde\pi}^{\tilde{P}} \hat{\*v} - \Bell \hat{\*v} )_s \relname{def}{=} \tilde{P}(s, \tilde\pi(a))\tr \hat{\*z}_{s, \tilde\pi(s)} - \min_{P \in \hat\ambset} P(s, \hat\pi(a))\tr \hat{\*z}_{s, \hat\pi(a)} \\
	&\relname{b}{\le} \tilde{P}(s, \tilde\pi(a))\tr \hat{\*z}_{s, \tilde\pi(s)} - \min_{P \in \hat\ambset} P(s, \tilde\pi(a))\tr \hat{\*z}_{s, \tilde\pi(a)} \\
	&\le \max_{a\in\actions} \left(\tilde{P}(s, a)\tr \hat{\*z}_{s, a} - \min_{P \in \hat\ambset} P(s, a)\tr \hat{\*z}_{s, a} \right)\\
	&\relname{c}{\le} \max_{a\in\actions} \left(\max_{P \in \hat\ambset} P(s, a)\tr \hat{\*z}_{s, a} - \min_{P \in \hat\ambset} P(s, a)\tr \hat{\*z}_{s, a} \right) \\
	&\le \max_{a\in\actions}\, \beta^{s,a}_{\hat{\*z}}(\*w, \psi)~.
	\end{align*}
	(a) follows from $\hat{\*v}$ being the fixed point of $\Bell$, (b) follows from the optimality of $\hat\pi$: $\hat\pi(s) \in \arg\max_{a\in\actions} \min_{\*p\in\hat\ambset_{s,a}} \*p\tr \*z_{s,a}$, and (c) follows from $\tilde{P} \in \hat\ambset$. The rest follows by algebraic manipulation. Applying the inequality above to all states, we get:
	\begin{equation} \label{eq:subopt_bellman}
	\BellT_{\tilde\pi}^{\tilde{P}} \hat{\*v} - \hat{\*v} \le \beta_{\hat{\*z}}(\*w, \psi) \cdot \one~.
	\end{equation}
	We can now use the standard dynamic programming bounding technique to bound $\norm{\tilde{\*v} - \hat{\*v}}_\infty$ as follows:
	\[
	\zero \relname{a}{\le} \tilde{\*v} - \hat{\*v}
	\relname{b}{=} \tilde{\*v} - \BellT_{\tilde\pi}^{\tilde{P}} \hat{\*v} + \BellT_{\tilde\pi}^{\tilde{P}} \hat{\*v} - \hat{\*v}
	\stackrel{\eqref{eq:subopt_bellman}}{\le} \tilde{\*v} - \BellT_{\tilde\pi}^{\tilde{P}} \hat{\*v} + \beta_{\hat{\*z}}(\*w, \psi) \cdot \one
	\relname{c}{\le} \BellT_{\tilde\pi}^{\tilde{P}}\tilde{\*v} - \BellT_{\tilde\pi}^{\tilde{P}} \hat{\*v} + \beta_{\hat{\*z}}(\*w, \psi) \cdot \one~.
	\]
	We have (a) because $\hat{\*v} \le \tilde{\*v}$ because $\Bell \tilde{\*v} \le \tilde{\*v}$ and thus $\tilde{\*v} \ge \Bell \Bell \tilde{\*v} \ge \ldots \ge \Bell \ldots \Bell \tilde{\*v} \ge \hat{\*v}$ because $\hat{\*v}$ is the fixed point of $\Bell$ and $\Bell$ is monotone. (b) we add $\zero$, (c) $\tilde{\*v}$ is the fixed point of $\BellT_{\tilde\pi}^{\tilde{P}}$.

	Next, apply $L_\infty$ norm to all sides, which is possible because the values are non-negative:
	\begin{align*}
		\norm{\tilde{\*v} - \hat{\*v}}_\infty &\le \norm{\BellT_{\tilde\pi}^{\tilde{P}}\tilde{\*v} - \BellT_{\tilde\pi}^{\tilde{P}} \hat{\*v} + \beta_{\hat{\*z}}(\*w, \psi) \cdot \one}_\infty
		\\
		\norm{\tilde{\*v} - \hat{\*v}}_\infty &\le \gamma \cdot \norm{\tilde{\*v} -  \hat{\*v}}_\infty + \beta_{\hat{\*z}}(\*w, \psi)  \\
		\norm{\tilde{\*v} - \hat{\*v}}_\infty &\le \beta_{\hat{\*z}}(\*w, \psi)  / (1-\gamma)~.
	\end{align*}
	The first step follows by triangle inequality, and the second step follows from $\BellT_{\tilde\pi}^{\tilde{P}}$ being a $\gamma$ contraction in the $L_\infty$ norm.

	To prove the bound on $y\opt$ and $\hat{v}$, we show that $y\opt \le \zeta$ where $\zeta = \hat{\rho} + \beta_{\hat{\*z}}(\*w,\psi) / (1-\gamma)$. Suppose to the contrary that $y\opt > \zeta$. Realize that $y\opt$ optimal in \eqref{eq:percentile_Bayesian} must satisfy:
	\begin{equation} \label{eq:subopt_ystar}
	\P_{\tilde{P} \sim f} \left[ \max_{\pi\in\Pi} \rho(\pi, \tilde{P}) \ge y\opt \right]  \ge 1-\delta~,
	\end{equation}
	because $\max_{\pi\in\Pi} \rho(\pi, \tilde{P}) \ge \rho(\pi\opt, \tilde{P})$ for $\pi\opt$ optimal in \eqref{eq:percentile_Bayesian}. Recall also that from the first part of the theorem:
	\begin{equation} \label{eq:suboptimal_zeta}
	\P_{\tilde{P} \sim f} \left[\max_{\pi\in\Pi} \rho(\pi, \tilde{P}) \ge \zeta \right]  \le \delta~.
	\end{equation}
	We now derive a contradiction as follows:
	\[
	\delta \stackrel{\eqref{eq:suboptimal_zeta}}{\ge}  \P_{\tilde{P} \sim f} \left[\max_{\pi\in\Pi} \rho(\pi, \tilde{P}) \ge \zeta \right] \relname{a}{\ge} \P_{\tilde{P} \sim f}\left[\max_{\pi\in\Pi} \rho(\pi, \tilde{P}) \ge y\opt \right] \stackrel{\eqref{eq:subopt_ystar}}{\ge} 1-\delta~.
	\]
	Here (a) follows from the assumption $y\opt > \zeta$. Then $\delta \ge 1-\delta$ is a contradiction with $\delta < 0.5$. Finally, $0 \le y\opt - \hat\rho$ follows directly from the optimality of $y\opt$ and \cref{thm:ambiguity_sufficient}, which proves the theorem.
\end{proof}

\subsection{Proof of Results in \cref{sec:shape}} \label{app:sec_shape}
\begin{proof}[Proof of \Cref{thm:choose_weights}]
We omit the $s,a$ subscripts to simplify the notation. By relaxing the non-negativity constraints on $\*p$ and using substitution $\*q_1 = \*p_1 - \bar{\*p}$ and $\*q_2 = \*p_2 - \bar{\*p}$, we get the following upper bound:
\begin{align*}
\beta_{\*z}^{s,a}(\*w, \psi)  &= \max_{\*p_1, \*p_2} \; \Bigl\{ (\*p_1 - \*p_2)\tr \*z \ss \*p_1,\*p_2 \in\ambset_{s,a}(\*w,\psi) \Bigr\} \\
&= \max_{\*p_1, \*p_2} \; \Bigl\{ (\*p_1 - \*p_2)\tr \*z \ss \norm{\*p_1 - \bar{\*p}}_{\*w} \le \psi,\, \norm{\*p_2 - \bar{\*p}}_{\*w} \le \psi,\, \*p_1 \in \Delta^S,\, \*p_2\in\Delta^S \Bigr\} \\
&\le \max_{\*p_1, \*p_2\in\Real^S} \; \Bigl\{ (\*p_1 - \*p_2)\tr \*z \ss \norm{\*p_1 - \bar{\*p}}_{\*w} \le \psi,\, \norm{\*p_2 - \bar{\*p}}_{\*w} \le \psi,\, \one\tr \*p_1 = 1,\, \one\tr \*p_2 = 1 \Bigr\} \\
&= \max_{\*q_1, \*q_2\in\Real^S} \; \Bigl\{ (\*q_1 - \*q_2)\tr \*z \ss \norm{\*q_1}_{\*w} \le \psi,\, \norm{\*q_2}_{\*w} \le \psi,\, \one\tr \*q_1 = 0,\, \one\tr \*q_2 = 0 \Bigr\} \\
&= \max_{\*q_1\in\Real^S} \; \Bigl\{ \*q_1 \tr \*z \ss \norm{\*q_1}_{\*w} \le \psi,\, \one\tr \*q_1 = 0 \Bigr\} + \max_{\*q_2\in\Real^S} \; \Bigl\{ \*q_2 \tr (-\*z) \ss \norm{\*q_2}_{\*w} \le \psi,\, \one\tr \*q_2 = 0\Bigr\}~.
\end{align*}
The last equality follows because the the optimization problems over $\*q_1$ and $\*q_2$ are independent. From the absolute homogeneity of the $\norm{\cdot}_{\*w}$ we have that:
\[
\max_{\*q_2\in\Real^S} \; \Bigl\{ \*q_2 \tr (-\*z) \ss \norm{\*q_2}_{\*w} \le \psi,\, \one\tr \*q_2 = 0\Bigr\} \;=\; \max_{\*q_2\in\Real^S} \; \Bigl\{ \*q_2 \tr \*z \ss \norm{\*q_2}_{\*w} \le \psi,\, \one\tr \*q_2 = 0\Bigr\}~,
\]
and therefore:
\[
\beta_{\*z}^{s,a}(\*w, \psi) \le 2\cdot \max_{\*q\in\Real^S} \; \Bigl\{ \*q \tr \*z \ss \norm{\*q}_{\*w} \le \psi,\, \one\tr \*q = 0 \Bigr\} ~.
\]
Substituting $\*q = \*p - \bar{\*p}$ we get:
\begin{equation} \label{eq:shape_optimization_bound}
\beta_{\*z}^{s,a}(\*w, \psi) \le 2\cdot \max_{\*p\in\Real^S} \; \Bigl\{ \*p \tr \*z \ss \norm{\*p - \bar{\*p}}_{\*w} \le \psi,\, \one\tr \*p = 1 \Bigr\} - 2 \cdot \*z\tr\bar{\*p} ~.
\end{equation}
We can reformulate the optimization problem on the right-hand side of \eqref{eq:shape_optimization_bound}, again using variable substitution $\*q = \*p - \bar{\*p}$:
\begin{equation*}
	\begin{aligned}
		\max_{\*q \in \Real^S} \quad & 2\cdot ( \*q + \bar{\*p} )\tr  \*z  - 2 \cdot \*z\tr\bar{\*p} \\
		\text{s.t. } \quad
		& \norm{\*q}_{\*w} \le \psi\\
		& \*1\tr ( \*q + \bar{\*p}) = 1 \implies  \*1\tr  \*q = 0~.
	\end{aligned}
\end{equation*}
Canceling out $\bar{\*p}\tr \*z $, we continue with:
\begin{equation*}
	\begin{aligned}
		2 \cdot \max_{\*q \in \Real^S} \quad &  \*q  \tr  \*z \\
		\text{s.t. } \quad
		& \norm{\*q}_{\*w} \le \psi\\
		& \*1\tr  \*q = 0~.  \\
	\end{aligned}
\end{equation*}
By applying the method of Lagrange multipliers, we obtain:
\begin{equation*}
	\begin{aligned}
		\min_{\lambda\in\Real} \max_{\*q\in\Real^S} \quad &   \*q  \tr  \*z  -  \lambda \cdot (\*q \tr\*1)  =   \*q \tr ( \*z  -  \lambda \cdot \*1)  \\
		\text{s.t. } \quad
		& \norm{\*q}_{\*w} \le \psi~.\\
	\end{aligned}
\end{equation*}
Letting $\*x =\frac{\*q}{\psi}$, we get:
\begin{equation*}
	\begin{aligned}
		\quad \min_{\lambda\in\Real}  \max_{\*x\in\Real^S} &\quad  \psi \cdot \*x  \tr (   \*z  -   \lambda \cdot \*1)  \\
		\text{s.t. } \quad
		& \norm{\*x}_{\*w} \le 1~.\\
	\end{aligned}
\end{equation*}
Given the definition of the \emph{dual norm}, $\norm{\*z}_{\star}=\sup\{\*z^{\intercal }  \*x \ss \norm{\*x} \leq 1\}$, we have:
\begin{equation*}
	\begin{aligned}
		\beta_{\*z}^{s,a}(\*w, \psi)
		&\le 2 \cdot \min_{\lambda \in\Real}  \, \psi\cdot \norm{ \*z - \lambda \cdot \*1}_{\star}  \\
	 	&\le 2 \cdot \psi \cdot \norm{ \*z - \lambda \cdot \*1}_{\star}  ~.\\
	\end{aligned}
\end{equation*}
\end{proof}

\begin{proof}[Proof of \cref{th:l1_l8}]
	Assume we are given a set of positive weights $\*w \in \real^n_{++}$ for the following weighted $L_1$ optimization problem:
	\begin{equation}
	\begin{aligned}
	\max_{\*x\in\Real^S} \quad & \*z \tr \*x \\
	\text{s.t. } \quad
	& \norm{\*x}_{1,\*w} \le 1~.\\
	\end{aligned} \label{eq:l1}
	\end{equation}
	We have:
	\begin{align*}
	\*x\tr \*z &= \sum_{i=1}^{n} x_i\cdot z_i \le \sum_{i=1}^{n} |x_i \cdot z_i|\\ &\overset{(a)}{\leq} \sum_{i=1}^{n} |x_i| \cdot |z_i|
	= \sum_{i=1}^{n} w_i\cdot|x_i| \cdot \frac{1}{w_i} \cdot |z_i| \\
	&\le \max_{i=1,\dots,n} \bigg\{ \frac{1}{w_i} \cdot |z_i|\bigg\} \cdot \sum_{i=1}^{n} w_i|x_i|
	= \max_{i=1,\ldots,n}\bigg\{ \frac{1}{w_i} \cdot |z_i|\bigg\} \cdot \lVert \*x \rVert_{1,\*w}\\
	&\overset{(b)}{\leq} \max_{i=1,\ldots,n}\bigg\{ \frac{1}{w_i} |z_i|\bigg\} = \lVert \*z \rVert_{\infty,\frac{1}{\*w}} ~.&&
	\end{align*} \label{eq:l1_dual}
	Here, (a) follows from the Cauchy-Schwarz inequality, and (b) follows from the constraint $\lVert \*x \rVert_{1,\*w} \le 1$ of \eqref{eq:l1}.
\end{proof}

\begin{proof}[Proof of \cref{prop:weights_analytical}]
	We use the notation $1/\*w$ to denote an elementwise inverse of $\*w$ such that $(1/\*w)_i = 1/w_i, i\in\states$. Note that for weighted $L_1$-constrained sets $q = \infty$, and for the $L_\infty$-constrained sets $q = 1$. The value $\bar{\lambda}$ in \eqref{eq:upperbound_negativity} is fixed ahead of time and does not change with $\*w$. Recall that the constraint $\sum_{i=1}^S w_i^2 = 1$ serves to normalize $\*w$ in order to preserve the desired robustness guarantees with \emph{the same} $\psi$. This is because scaling both $\*w$ and $\psi$ simultaneously by an identical factor leaves the ambiguity set unchanged. We adopt the constraint from an approximation of the guarantee by linearization of the upper bound using Jensen’s inequality.
	Next, omitting terms that are constant with respect to $\*w$ simplifies the optimization to:
	\begin{equation} \label{eq:obj_reformulated}
		\*w\opt \in \argmin_{\*w \in \real^S_{++}} \left\{ \norm{ \*z - \bar\lambda \*1}_{q, \frac{1}{\*w}} ~:~ \sum_{i=1}^S w_i^2 = 1 \right\} ~.
	\end{equation}
	For $q = \infty$, the nonlinear optimization problem in \eqref{eq:obj_reformulated} is convex and can be solved \emph{analytically}. Let $b_i = \abs{z_i - \bar{\lambda}}$ for $i=1,\ldots, S$, then \eqref{eq:obj_reformulated} turns to:
	\begin{equation} \label{eq:compute_weight_analytically_l1}
		\min_{t, \*w \in\real^S_{++}}  \left\{ t ~:~ t \ge b_i/w_i,\, \sum_{i=1}^S w_i^2 = 1 \right\}~.
	\end{equation}
	The constraints $\*w > \zeros$ cannot be active since otherwise  $1/w_i$ results in undefined division by zero and can be safely ignored. Then, the convex optimization problem in \cref{eq:compute_weight_analytically_l1} has a linear objective, $S+1$ variables ($\*w$'s and $t$), and $S+1$ constraints. All constraints are active, therefore, in the optimal solution $\*w\opt$~\citep{Bertsekas2003} which must  satisfy:
	\begin{equation} \label{eq:analytical_weights_l1}
		w_i\opt = \nicefrac{b_i}{\sqrt{\sum_{j=1}^S b_j^2} } ~.
	\end{equation}
	Since $\sum_i w_i^2 = 1$ implies $\sum_i b_i^2 / t^2 =1$, we  conclude that $t = \sqrt{\sum_i b_i^2}$. For $q = 1$, the equivalent optimization of \eqref{eq:compute_weight_analytically_l1} becomes:
	\begin{equation} \label{eq:compute_weight_analytically_l8}
		\min_{\*w > \zeros} \; \left\{ \sum_{i=1}^S b_i / w_i ~:~ \sum_{i = 1}^S w_i^2 = 1 \right\}~.
	\end{equation}
	Again, the inequality constraints on weights $\*w > 0$ can be relaxed. Using the necessary optimality conditions (and a Lagrange multiplier), one solution for the optimal weights $\*w$ are:
	\begin{equation} \label{eq:analytical_weights_l8}
		w_i\opt = \nicefrac{b_i^{1/3}}{\sqrt{\sum_{j=1}^S {b_j^{2/3}}} } ~.
	\end{equation}
\end{proof}

\subsection{Proof of Results in \cref{sec:size}} \label{app:sec_size}

\begin{proof}[Proof of \cref{prop:size_correct}]
	The algorithm is an instance of the Sample Average Approximation~(SAA) scheme. The result, therefore, is a direct consequence of Theorem~4.2 in \cite{petrik2019beyond} and
	Theorem~5.3 in \cite{Shapiro2014}.
\end{proof}

\subsection{Proof of Results in \cref{sec:frequentist}} \label{app:sec_freq}

We need several auxiliary results before proving the results.
\begin{theorem}[Weighted $\linf$ error bound (Hoeffding)]
	\label{thm:weighted_linf}
	Suppose that $\bar{\*p}\sa$ is the empirical estimate of the transition probability obtained from $n\sa$ samples for some $s \in \states$ and $ a \in \actions$. Then:
	\begin{equation} \label{eq:w8_error}
	\prob_{\bar{\*p}\sa} \left[ \norm{\bar{\*p}\sa - \*p\opt\sa}\liw \geq \psi\sa \right] \leq 2 \sum_{i=1}^S \exp \left(-2 \frac{\psi\sa^2 n\sa}{w_{i}^2} \right)~.
	\end{equation}
\end{theorem}
\begin{proof}
	First, we will express the weighted $\linf$ distance between two distributions $\bar{\*p}$ and $\*p^\star$ in terms of an optimization problem. Let $\*1_i \in \real^\states$ be the indicator vector for an index $i \in \states$:
	\begin{align*}
	\norm{\bar{\*p}\sa - \*p^\star\sa }\liw  &= \max_{\*z} \left\lbrace \*z \tr W(\bar{\*p}\sa - \*p^\star\sa) : \norm{\*z} _1 \leq 1 \right\rbrace \\
	&= \max_{i \in \states} \Bigl\lbrace \*1_i W(\bar{\*p}\sa - \*p^\star\sa), - \*1_i W(\bar{\*p}\sa - \*p^\star\sa) \Bigr\rbrace ~.
	\end{align*}
	Here, weights are on the diagonal entries of $W$. Using the expression above, we can bound the probability in the lemma as follows:
		\begin{align*}
		\prob  \left[ \norm{\bar{\*p}\sa - \*p^\star\sa}\liw \geq \psi  \right]  &=
		\prob \left[ \max_{i \in \states } \left\lbrace \*1_i W(\bar{\*p}\sa - \*p^\star\sa), - \*1_i W(\bar{\*p}\sa - \*p^\star\sa) \right\rbrace \geq \psi\sa \right] \\
		& \stackrel{(a)}{\leq}  S \max_{i \in \states} \prob \left[ \*1_i W (\bar{\*p}\sa - \*p^\star\sa) \geq \psi\sa \right] +  S \max_{i \in \states} \prob \left[- \*1_i W(\bar{\*p}\sa - \*p^\star\sa) \geq \psi\sa \right] \\
		& \stackrel{(b)}{\leq} 2 \sum_{i=1}^S \exp \left(-2\frac{\psi\sa^2 n}{w_{i}^2} \right)~.
		\end{align*}
	Here, $(a)$ follows from union bound, and $(b)$ follows from Hoeffding's inequality since $\*1 \tr_i \bar{\*p} \in [0,1]$ for any $i \in \states$ and its mean is $\*1_i \tr \*p\opt$.
\end{proof}

Now we describe a proof of error bound in~\eqref{eq:w1_error} on the weighted $L_1$  distance between the estimated transition probabilities $\bar{\*p}$ and the true one $\*p^\star$ over each state $s \in \states = \{ 1, \ldots, S \}$ and action $a \in \actions = \{ 1 , \ldots , A \}$. The proof is an extension to Lemma C.1 ($L_1$ error bound) in~\cite{petrik2019beyond}.

\begin{theorem}[Weighted $L_1$ error bound (Hoeffding)]
	\label{thm:weighted_lone}
	Suppose that $\bar{\*p}\sa$ is the empirical estimate of the transition probability obtained from $n\sa$ samples for some $s \in \states$ and $ a \in \actions$. If the weights $\*w \in \real_{++}^S$ are sorted in a non-increasing order $w_i \ge w_{i+1}$, then:
	\begin{equation} \label{eq:w1_error}
		\prob_{\bar{\*p}\sa} \left[ \norm{\bar{\*p}\sa - \*p\opt\sa}_{1,\*w}  \geq \psi\sa  \right] \leq
		2\sum_{i = 1}^{S-1} 2^{S - i} \exp  \left(  -  \frac{\psi\sa^2n\sa}{2 w_i^2} \right)~.
	\end{equation}
\end{theorem}
\begin{proof}
	Let $\*q\sa = \bar{\*p}\sa - \*p^\star\sa$. To shorten notation in the proof, we omit the $s, a$ indexes when there is no ambiguity. We assume that all weights are non-negative. First, we will express the $L\lw$ norm of $\*q$ in terms of an optimization problem. It is worth noting that $\*1\tr \*q =0$.  Let $\*1_{\qeu_1}, \*1_{\qeu_2} \in \real^\states$ be the indicator vectors for some subsets $\qeu_1,\qeu_2 \subset \states$ where $\qeu_2 = \states \setminus \qeu_1$. According to \cref{th:l1_l8} we have:
	\begin{align*}
	\norm{\*q}_{1,w} &= \max_\*z \left\lbrace  \*z\tr \*q : \norm{\*z}_{\infty,\frac{1}{w}} \leq 1 \right\rbrace    \\
	&= \max_{\qeu_1,\qeu_2 \in 2^\states} \left\lbrace \*1_{\qeu_1} \tr W \*q +  \*1_{\qeu_2}\tr W (-\*q)  :  \qeu_2 = \states \setminus \qeu_1 \right\rbrace ~.
	\end{align*}
	Here weights are on the diagonal entries of $W$. Using the expression above, we can bound the probability as follows:
		\begin{align*}
		\prob \left[ \max_{\qeu_1,\qeu_2 \in 2^\states} \left\lbrace \*1_{\qeu_1}\tr W \*q + \*1_{\qeu_2}\tr W (-\*q) \right\rbrace  \geq \psi \right]
		& \stackrel{(a)}{\leq} \prob \left[ \max_{{\qeu_1} \in 2^\states} \left\lbrace \*1_{\qeu_1}\tr W \*q  \right\rbrace  \geq \frac{\psi}{2} \right] + \prob \left[ \max_{{\qeu_2} \in 2^\states} \left\lbrace \*1_{\qeu_2}\tr W (-\*q)  \right\rbrace  \geq \frac{\psi}{2} \right]  \\
		& \leq \sum_{{\qeu_1} \in 2^{\states}}  \prob \left[ \*1_{\qeu_1} \tr W \*q  \geq \frac{\psi}{2}   \right]  + \sum_{{\qeu_2} \in 2^{\states}}  \prob \left[ \*1_{\qeu_2} \tr W (-\*q)  \geq \frac{\psi}{2}  \right]  \\
		& = \sum_{{\qeu_1} \in 2^{\states}}  \prob \left[ \*1_{\qeu_1} \tr W (\bar{\*p} - \*p^\star)  \geq \frac{\psi}{2}   \right] + \sum_{{\qeu_2} \in 2^{\states}}  \prob \left[ \*1_{\qeu_2} \tr W (- \bar{\*p} + \*p^\star)  \geq \frac{\psi}{2}   \right]   \\
		& \stackrel{(b)}{\leq} \sum_{{\qeu_1} \in 2^{\states}} \exp \left( - \frac{\psi^2 n }{2 \norm{\*1_{\qeu_1} \tr W}_\infty^2}\right) + \sum_{{\qeu_2} \in 2^{\states}} \exp \left( -
		\frac{\psi^2 n }{2 \norm{\*1_{\qeu_2} \tr W}_\infty^2}\right)   \\
		& \stackrel{(c)}{=} 2 \sum_{i = 1}^{S -1} 2^{S - i} \exp  \left( - \frac{\psi^2n}{2 w_i^2} \right)~.
		\end{align*}
	$(a)$ follows from union bound, and $(b)$ follows from Hoeffding's inequality. $(c)$ follows by $\qeu_1 ^c = \qeu _2$ and sorting weights $\*w = \{w_1, \ldots, w_n \}$ in non-increasing order.
\end{proof}


\begin{proof}[Proof of \cref{thm:budget_freq_infty}]
	The result follows from Lemma~A.1 in \cite{petrik2019beyond} and \cref{thm:weighted_linf} by algebraic manipulation.
\end{proof}

\begin{proof}[Proof of \cref{thm:budget_freq_l1}]
	The result follows from Lemma~A.1 in \cite{petrik2019beyond} and \cref{thm:weighted_lone} by algebraic manipulation.
\end{proof}

\subsection{Bernstein Concentration Inequalities} \label{app_bernstein}

\begin{theorem}[Weighted $L_1$ error bound (Bernstein)]
	\label{thm:weighted_lone_bern}
	Suppose that $\bar{\*p}\sa$ is the empirical estimate of the transition probability obtained from $n\sa$ samples for some $s \in \states$ and $ a \in \actions$. If the weights $\*w \in \real_{++}^S$ are sorted in non-increasing order $w_i \ge w_{i+1}$, then the following holds when using Bernstein's inequality:
	\[
	\prob \left[ \norm{\bar{\*p}\sa - \*p^\star\sa}_{1,\*w}  \geq \psi\sa  \right] \leq   2 \sum_{i = 1}^{S-1} 2^{S - i} \exp  \left( - \frac{3\psi^2n}{6 w_i^2 + 4\psi w_i} \right)
	\]
	where $\*w \in \real_{++}^S$ is the vector of weights. The weights are sorted in non-increasing order.
\end{theorem}
\begin{proof} The proof is similar to the proof of \cref{thm:weighted_lone} until section $b$. The proof continues from section $(b)$ as follows:
		\begin{align*}
		&\stackrel{(b)}{\leq} \sum_{{\qeu_1} \in 2^{\states}} \exp \left( - \frac{3\psi^2 n}{24\sigma^2+ 4 c \psi}\right) + \sum_{{\qeu_2} \in 2^{\states}} \exp \left( - \frac{3\psi^2 n}{24\sigma^2+ 4 c \psi}\right)   \\
		& \stackrel{(c)}{\leq} \sum_{{\qeu_1} \in 2^{\states}} \exp \left( - \frac{3\psi^2 n}{6\norm{\*1_{\qeu_1} \tr W}_\infty^2+4\psi\norm{\*1_{\qeu_1} \tr W}_\infty}\right) + \sum_{{\qeu_2} \in 2^{\states}} \exp \left( - \frac{3\psi^2 n}{6\norm{\*1_{\qeu_2} \tr W}_\infty^2+4\psi\norm{\*1_{\qeu_2} \tr W}_\infty}\right)   \\
		& \stackrel{(d)}{=} 2 \sum_{i = 1}^{S - 1} 2^{S - i} \exp  \left( - \frac{3\psi^2n}{6 w_i^2 + 4\psi w_i} \right)~.
		\end{align*}
	Here $(b)$ follows from Bernstein's inequality where $\sigma^2$ is the mean of variance of random variables, and $c$ is their upper bound~\citep{devroye2013probabilistic}. In the weighted case, with conservative estimate of variance $\sigma^2 = \norm{\*1_{\qeu_1} \tr W}_\infty^2 / 4$, and $c = \norm{\*1_{\qeu_1} \tr W}_\infty$,  because the random variables are drawn from \emph{Bernoulli} distribution with the maximum possible variance of $1/4$. $(d)$ follows by sorting weights $\*w $ in non-increasing order.
\end{proof}

\section{Detailed Experimental Results} \label{app:empirical}

\subsection{Experimental Setup}

We assess $L_1-$ and $L_\infty$-bounded ambiguity sets, both with weights and without weights. We compare Bayesian credible regions with frequentist Hoeffding- and Bernstein-style sets. We start by assuming a true underlying model that produces simulated datasets containing $20$ samples for each state and action. The frequentist methods construct ambiguity sets directly from the datasets. Bayesian methods combine the data with a prior to compute a posterior distribution and then draw $20$ samples from the posterior distribution to construct a Bayesian ambiguity set.

\subsection{RiverSwim MDP Graph}

\begin{figure}[H]
	\begin{center}
		\begin{tikzpicture}
		[>=stealth',
		scale=2.50,
		shorten > = 1pt,
		node distance = 1.7cm,
		el/.style = {inner sep=1pt, align=left, sloped},
		every label/.append style = {font=\tiny}
		]
		\node (q0) [state,thick,inner sep=1pt,minimum size=0.5pt ]     {$s_0$};
		\node (q1) [state,thick,right=of q0,inner sep=1pt,minimum size=0.5pt]   {$s_1$};
		\node (q2) [state,line width=0pt, draw=white,right=of q1,inner sep=1pt,minimum size=0.5pt]   {$\cdots$};
		\node (q4) [state,thick,right=of q2,inner sep=1pt,minimum size=0.5pt]   {$s_4$};
		\node (q5) [state,thick,right=of q4,inner sep=1pt,minimum size=0.5pt]   {$s_5$};
		\path[->]
		(q0)  edge [in=260,out=280,loop, dashed] node[el,below, font=\tiny] {$(1,r=5)$}   (q0)
		(q0)  edge [in=80,out=100,loop] node[el,above, font=\tiny] {$0.7$}                    (q0)
		(q1)  edge [in=80,out=100,loop] node[el,above, font=\tiny] {$0.6$}                    (q1)
		(q0)  edge [out=60, bend right=-50, in=150]  node[el,above, font=\tiny]  {$0.3$}      (q1)
		(q1)  edge [out=90,bend left=10, in=215]  node[el,below, font=\tiny]  {$0.1$}        (q0)
		(q1)  edge [bend left=50, dashed]  node[el,below, font=\tiny]  {$1$}                  (q0)
		(q2)  edge [in=80,out=100,loop] node[el,above, font=\tiny] {$0.6$}                    (q2)
		(q1)  edge [out=60, bend right=-50, in=150]  node[el,above, font=\tiny]  {$0.3$}     (q2)
		(q2)  edge [out=90,bend left=10, in=215]  node[el,below, font=\tiny]  {$0.1$}        (q1)
		(q2)  edge [bend left=50, dashed]  node[el,below, font=\tiny]  {$1$}                  (q1)
		(q4)  edge [in=80,out=100,loop] node[el,above, font=\tiny] {$0.6$}                    (q4)
		(q2)  edge [out=60, bend right=-50, in=150]  node[el,above, font=\tiny]  {$0.3$}     (q4)
		(q4)  edge [out=90,bend left=10, in=215]  node[el,below, font=\tiny]  {$0.1$}        (q2)
		(q4)  edge [bend left=50, dashed]  node[el,below, font=\tiny]  {$1$}                  (q2)
		(q5)  edge [in=80,out=100,loop] node[el,above, font=\tiny] {$(0.3, r=10000)$}                    (q5)
		(q4)  edge [out=60, bend right=-50, in=150]  node[el,above, font=\tiny]  {$0.3$}     (q5)
		(q5)  edge [out=90,bend left=10, in=215]  node[el,below, font=\tiny]  {$0.7$}        (q4)
		(q5)  edge [bend left=50, dashed]  node[el,below, font=\tiny]  {$1$}                  (q4);
		\end{tikzpicture}
	\end{center}
	\caption{RiverSwim problem with six states and two actions (left-dashed arrow, right-solid arrow). The agent starts in either states $s_1$ or $s_2$.}\label{fig:riverswim}
\end{figure}

\subsection{Full Empirical Results}


\cref{tab:riverswim,tab:machine_replacement,tab:population,tab:inventory,tab:population} report the high-confidence lower bound on the return for the domains that we investigate. The column denotes the confidence $1-\delta$ and the algorithm used to compute the weights $\*w$ for the ambiguity set: ``Unif.w'' corresponds to $\*w = \one$, ``Analyt.w'' corresponds to weights computed by \cref{alg:weight_optimization}, and ``SOCP.w'' corresponds to weights computed by solving \eqref{eq:socp}. The rows indicate which norm was used to define the ambiguity set ($L_1$ or $L_\infty$) and whether Bayesian (B) or frequentist (H) guarantees were used. Note that the SOCP formulation is limited to the $L_1$ ambiguity sets.

\begin{table*}[!htbp]
	\centering
	\begin{tabularx}{0.7\textwidth}{lcccccc}
		\toprule
		\multirow{2}[7]{*}{Method} & \multicolumn{3}{c}{$\delta = 0.5$} & \multicolumn{3}{c}{$\delta = 0.05$} \\
		\cmidrule(l){2-4} \cmidrule(l){5-7}
		& Unif.w & Analyt.w & SOCP.w & Unif.w & Analyt.w & SOCP.w \\
		\midrule
		$L_1 B$ & 33887 & \textbf{51470} & 48620 & 25252 & \textbf{47284} & 43504 \\
		$\linf B$ &33887 &  \textbf{48258} & - & 25252 &\textbf{43247}&- \\
		\midrule
		$L_1$ H& 16354 & \textbf{33116} & 30268 & 12555 & \textbf{29472} & 26398 \\
		$\linf$ H & 20055 & \textbf{40166} & - & 15184 & \textbf{35955} & - \\
		\bottomrule
	\end{tabularx}
	\caption{The return with performance guarantees for the RiverSwim experiment. The return of the nominal MDP is 63080. } \label{tab:riverswim}
\end{table*}

\begin{table*}[!htbp]
	\centering
	\begin{tabularx}{0.7\textwidth}{lcccccc}
		\toprule
		\multirow{2}[7]{*}{Method} & \multicolumn{3}{c}{$\delta = 0.5$} & \multicolumn{3}{c}{$\delta = 0.05$} \\
		\cmidrule(l){2-4} \cmidrule(l){5-7}
		& Unif.w & Analyt.w & SOCP.w & Unif.w & Analyt.w & SOCP.w \\
		\midrule
		$L_1 B$ & -38.1 & \textbf{-22.7} & -26.8 & -42.0 & \textbf{-23.7} & -28.4\\
		$\linf B$ & -38.1 & \textbf{-22.6}  & - & -42.0 & \textbf{-23.5}& - \\
		\midrule
		$L_1$ H& -86.8 & \textbf{-33.2} & -47.9 & -115.0 & \textbf{-34.5} & -53.1 \\
		$\linf$ H & -62.9 & \textbf{-29.5} & - & -74.8 & \textbf{-32.6} & - \\
		\bottomrule
	\end{tabularx}
	\caption{The return with performance guarantees for the Machine Replacement experiment. The return of the nominal MDP is -16.79. } \label{tab:machine_replacement}
\end{table*}

\begin{table*}[!htbp]
	\centering
	\begin{tabularx}{0.7\textwidth}{lcccccc}
		\toprule
		\multirow{2}[7]{*}{Method} & \multicolumn{3}{c}{$\delta = 0.5$} & \multicolumn{3}{c}{$\delta = 0.05$} \\
		\cmidrule(l){2-4} \cmidrule(l){5-7}
		& Unif.w & Analyt.w & SOCP.w & Unif.w & Analyt.w & SOCP.w \\
		\midrule
		$L_1 B$ & -25706 & \textbf{-12151} & -12668 & -25741 & \textbf{-12200} & -12704 \\
		$\linf B$ & -26782 & \textbf{-15468}  & - & -26795 & \textbf{-15623} & - \\
		\midrule
		$L_1$ H& -27499 & \textbf{-27034} & -27409 & -27501 & \textbf{-27047} & -27421 \\
		$\linf$ H & -27465 & \textbf{-27143} & - & -27473 & \textbf{-27184} & - \\
		\bottomrule
	\end{tabularx}
	\caption{The return with performance guarantees for the Population experiment. The return of the nominal MDP is -4127. } \label{tab:population}
\end{table*}

\begin{table*}[!htbp]
	\centering
	\begin{tabularx}{0.7\textwidth}{lcccccc}
		\toprule
		\multirow{2}[7]{*}{Method} & \multicolumn{3}{c}{$\delta = 0.5$} & \multicolumn{3}{c}{$\delta = 0.05$} \\
		\cmidrule(l){2-4} \cmidrule(l){5-7}
		& Unif.w & Analyt.w & SOCP.w & Unif.w & Analyt.w & SOCP.w \\
		\midrule
		$L_1 B$ & 3.75 & \textbf{15.7} & 10.9 & 3.64 & \textbf{15.0} & 10.6 \\
		$\linf B$ & 3.04 & \textbf{20.2}  & - & 2.87 & \textbf{19.8} & - \\
		\midrule
		$L_1$ H& -8.91 & \textbf{1.58} & -6.18  & -8.94 & \textbf{0.89} & -7.74 \\
		$\linf$ H & -8.37 & \textbf{5.83} & - & -8.63 & \textbf{4.90} & - \\
		\bottomrule
	\end{tabularx}
	\caption{The return with performance guarantees for the Inventory Management experiment. The return of the nominal MDP is 163.1. } \label{tab:inventory}
\end{table*}

\begin{table*}[!htb]
	\centering
	\begin{tabularx}{0.7\textwidth}{lcccccc}
		\toprule
		\multirow{2}[7]{*}{Method} & \multicolumn{3}{c}{$\delta = 0.5$} & \multicolumn{3}{c}{$\delta = 0.05$} \\
		\cmidrule(l){2-4} \cmidrule(l){5-7}
		& Unif.w & Analyt.w & SOCP.w & Unif.w & Analyt.w & SOCP.w \\
		\midrule
		$L_1 B$ & 3.83 & \textbf{8.28} & 4.21 & 3.82 & \textbf{8.25} & 4.20 \\
		$\linf B$ & 3.81 & \textbf{7.78}  & - & 3.78 & \textbf{7.71} & - \\
		\midrule
		$L_1$ H& 2.81 & \textbf{3.44} & 2.87  & 2.80 & \textbf{3.42} & 2.85 \\
		$\linf$ H & 3.18 & \textbf{3.94} & - & 3.15 & \textbf{3.92} & - \\
		\bottomrule
	\end{tabularx}
	\caption{The return with performance guarantees for the Cart-Pole experiment. The return of the nominal MDP is 11.11.} \label{tab:cartpole}
\end{table*}

\end{document}